\documentclass{article}

\usepackage[preprint]{neurips_2023}

\usepackage{hyperref}       
\usepackage{url}
\usepackage[table,xcdraw]{xcolor}  
\usepackage{graphicx}
\usepackage{natbib}
\usepackage{booktabs}      
\usepackage{amsfonts, amsmath, amsthm}       
\usepackage{nicefrac}       
\usepackage{microtype}     
\usepackage{xcolor}
\usepackage[symbol]{footmisc}         
\usepackage{paralist}
\usepackage{comment}
\usepackage[linesnumbered,ruled,vlined]{algorithm2e}
\usepackage{array, booktabs, multicol, multirow, makecell, float} 
\usepackage{enumitem}
\title{Formatting Instructions For NeurIPS 2023}

\author{
    Ali Rasekh$^{\text{*}}$\\
    Leibniz University Hannover, Germany\\
    L3S Research Center\\
    \texttt{ali.rasekh@L3S.de}\\
    \And
    Sepehr Kazemi Ranjbar$^{\text{*}}$\\
    Independent Researcher\\
    \texttt{sepehrkazemi9@gmail.com} \\
    \And
    Simon Gottschalk\\
    Leibniz University Hannover, Germany\\
    L3S Research Center\\
    \texttt{gottschalk@L3S.de}\\   
}

\newtheorem{definition}{Definition}
\newtheorem{theorem}{Theorem}
\newtheorem{assumption}{Assumption}

\definecolor{Gray}{gray}{0.85}
\newcolumntype{a}{>{\columncolor{Gray}}c}

\title{Multi-Rationale Explainable Object Recognition via Contrastive Conditional Inference}

\begin{document}

\maketitle

\begingroup
  \renewcommand\thefootnote{*}%
  \footnotetext{Equal contribution}%
\endgroup

\begin{abstract}
Explainable object recognition using vision-language models such as CLIP involves predicting accurate category labels supported by rationales that justify the decision-making process. Existing methods typically rely on prompt-based conditioning, which suffers from limitations in CLIP's text encoder and provides weak conditioning on explanatory structures. Additionally, prior datasets are often restricted to single, and frequently noisy, rationales that fail to capture the full diversity of discriminative image features. In this work, we introduce a multi-rationale explainable object recognition benchmark comprising datasets in which each image is annotated with multiple ground-truth rationales, along with evaluation metrics designed to offer a more comprehensive representation of the task. To overcome the limitations of previous approaches, we propose a contrastive conditional inference (CCI) framework that explicitly models the probabilistic relationships among image embeddings, category labels, and rationales. Without requiring any training, our framework enables more effective conditioning on rationales to predict accurate object categories. Our approach achieves state-of-the-art results on the multi-rationale explainable object recognition benchmark, including strong zero-shot performance, and sets a new standard for both classification accuracy and rationale quality. Together with the benchmark, this work provides a more complete framework for evaluating future models in explainable object recognition. The code will be made available online.
\end{abstract}

\section{Introduction}

The increasing use of contrastive vision-language models such as CLIP \cite{radford2021learning} in multimodal retrieval tasks \cite{alpay2023multimodal, lin2024mm}—including object recognition \cite{chen2022gscorecam, novack2023chils, abdelfattah2023cdul}—has raised concerns about their reliability and explainability in sensitive applications such as tumour detection \cite{liu2023clip, liu2024universal} and object recognition in autonomous driving \cite{elhenawy2025vision, prajwal2021object}.

The goal of explainable object recognition \cite{hui2020uncertainty, pintelas2020explainable} is to build models that not only predict a category label for a given image but also generate one or more human-readable rationales to justify their decisions (e.g., \textit{“A photo of a $<$cat$>$ because there are $<$large round eyes$>$, $<$furry body$>$, and $<$triangle face$>$”}). For contrastive models such as CLIP, it is common to assume access to predefined sets of rationales and categories \cite{Mao2023Doubly, ecor}. Given an input image, the model must first assign appropriate rationales, i.e., $r_1, r_2, \dots, r_M$ (e.g., \textit{$<$large round eyes$>$, $<$furry body$>$, $<$triangle face$>$}) to the image $x$, and then predict the category $c$ (e.g., \textit{cat}) based on both the image and its associated rationales. This process can be represented using two conditional distributions:

\begin{inparaenum}[(i)]
    \item $P(r_1, r_2, \dots, r_M|x)$, modelling the distribution of rationales given the image, and
    
    \item $P(c|x, r_1, r_2, \dots, r_M)$, modelling the distribution of the category conditioned on the image and rationales.
\end{inparaenum}

The Doubly Right Object Recognition (DROR) benchmark \cite{Mao2023Doubly} provides a category $c$ and only one rationale $r$ per image $x$ -- a key limitation, which constrains the model’s explainability. Moreover, DROR does not model $P(r|x)$, a limitation later addressed by ECOR \cite{ecor}. Both DROR and ECOR model $P(c|r,x)$ by feeding prompts of the form \textit{"A photo of a $<c>$ because there is $<r>$"} into the CLIP text encoder, assuming the word \textit{"because"} creates an effective conditioning of the category on the rationale. However, as we show in this paper, CLIP’s text encoder struggles to handle such conditioning, likely due to its training on noisy image-text pairs. Additionally, this prompt-based conditioning fails to scale to multi-rationale scenarios, since CLIP is known to perform poorly on long text inputs \cite{zhang2024long}.

Our contributions in this paper are as follows:
\begin{itemize}
    \item We extend the DROR benchmark to a multi-rationale setting by augmenting each image with multiple rationales and generalising the associated evaluation metric.
    \item To address the conditioning limitations, we analyse the CLIP embedding space theoretically and empirically, and propose a novel contrastive training-free conditional inference (CCI) framework. Unlike traditional CLIP inference, which models $P(c|x)$ by projecting the category embedding $c$ onto the image embedding $x$, we construct a hyperplane that integrates all conditioning information (image and rationales), thereby generalising inference from a single vector to a subspace in the embedding space.
    \item Finally, we demonstrate through extensive quantitative and qualitative experiments that our method achieves state-of-the-art performance, outperforming previous approaches in both single-dataset and cross-dataset evaluations. In particular, our performance in the zero-shot setting even surpasses supervised methods.
\end{itemize}

\section{Related Work} \label{sec:related}

\subsection{Explainability Approaches in Visual Recognition}
The field of explainable visual recognition has evolved along several complementary directions. Traditional approaches include gradient-based techniques such as Grad-CAM and its extensions \cite{selvaraju2017grad, Chattopadhay2018GradCAM++}, integrated gradients \cite{Sundararajan2017Integrated}, and SmoothGrad \cite{smilkov2017smoothgrad}, which identify important image regions by analyzing gradient flows. Perturbation-based methods like LIME \cite{ribeiro2016should} and RISE \cite{petsiuk2018rise} systematically modify inputs to assess feature importance. However, these techniques often struggle to provide semantically meaningful explanations for complex vision-language models.
More recent advances have focused on attribute and concept-based explanations. Pioneering datasets like AwA \cite{lampert2009learning} and CUB \cite{WahCUB_200_2011, michieli2019incremental} incorporate human-annotated attributes that serve as intermediate explanatory features. Building on this foundation, Concept Bottleneck Models \cite{pmlr-v119-koh20a} enforce a modular structure where classification decisions explicitly depend on predicted high-level concepts. Similarly, TCAV \cite{pmlr-v80-kim18d} quantifies concept influence through directional sensitivity analysis, while ProtoPNet \cite{chen2019looks} grounds decisions in prototypical examples. These methods often need fine-tuning to provide promising performance, while our approach, without any training, achieves state of the art performance even compared to supervised approaches.

\subsection{Text-Based and Multimodal Explanations}
The latest developments in explainable object recognition have embraced natural language as a flexible medium for expressing model rationales. \cite{Mao2023Doubly} introduced the "Doubly Right" benchmark, requiring models to generate both correct labels and accurate verbal justifications for their decisions \cite{Mao2023Doubly}. Their approach employs "why prompts" to elicit explanations from vision-language models, though it faces challenges in properly conditioning rationale generation. Similarly, \cite{ecor} proposed ECOR, which models explainability as a joint distribution over categories and rationales. While effective for zero-shot scenarios, this method encounters limitations with the CLIP text encoder since it utilizes the negations, and other types of semantic words (e.g., \textit{because}) that can be understood by LLMs, but not by the CLIP text encoder \cite{kamath2023text, chen2023difference}. Also in \cite{bhalla2024interpreting}, a zero-shot approach for sparse projection of image embedding on concept embeddings was introduced, which needs to solve an optimization problem explicitly.

Multimodal explanation approaches combine visual attention mechanisms with textual descriptions to enhance interpretability. \cite{Liu2023LLaVA} introduced LLaVA, leveraging multimodal attention to generate explanations for visual decisions \cite{Liu2023LLaVA}, while \cite{Zhang2024MMCoT} proposed MM-CoT, using multimodal chain-of-thought reasoning for explainable visual understanding \cite{Zhang2024MMCoT}. Approaches like X-VLM \cite{Zeng2022X2VLM} incorporate region-text alignment for more interpretable results, and e-ViL \cite{Kayser2021EViL} proposes a visual-entailment framework for generating explanations in vision-language tasks. These methods represent significant progress but still face challenges in generating explanations that are both faithful to the model's decision process and grounded in specific visual attributes. Further, these approaches typically rely on a single rationale per image, while we target identifying multiple rationales.

\section{Multi-Rationale Benchmark}\label{sec:MRDG}

In this section, we propose an automated pipeline to transform single-rationale object recognition datasets \cite{Mao2023Doubly} into multi-rationale versions and extend the evaluation metric accordingly.

\subsection{Extending Datasets}
In the DROR benchmark \cite{Mao2023Doubly}, each dataset (CIFAR-10$^+$, CIFAR-100$^+$, Food-101$^+$, Caltech-101$^+$, SUN$^+$, and ImageNet$^+$)\footnote{These datasets share only category names with their popular counterparts—CIFAR-10 \cite{krizhevsky2009learning}, CIFAR-100 \cite{krizhevsky2009learning}, Food-101 \cite{bossard2014food}, Caltech-101 \cite{li_andreeto_ranzato_perona_2022}, SUN \cite{xiao2010sun}, and ImageNet \cite{5206848}. The images are sourced from Google Image Search using queries of the form \textit{"A photo of $<$category$>$ which has $<$rationale$>$"}.} consists of samples in the triplet form $(x, c, r)$, where $x$ is an image, $c$ is the image category, and $r$ is the rationale associated with $c$ and $x$.

Although categories were annotated with $\sim$7 rationales, only one was assigned per image. We extend the number of rationales per image by using an open-vocabulary object detector \cite{NEURIPS2023_e6d58fc6}, retaining those with confidence above 0.1 (as recommended in \cite{NEURIPS2023_e6d58fc6}). This increases the average number of rationales per image from $1$ to an average of $3$, which is sufficient to generate meaningful explanations. The whole process is under human supervision to mitigate any mistakes or specific biases caused by the object detector in identifying rationales. We refer to the resulting multi-rationale datasets as CIFAR-10$^\dagger$, CIFAR-100$^\dagger$, Food-101$^\dagger$, Caltech-101$^\dagger$, SUN$^\dagger$, and ImageNet$^\dagger$. Appendix \ref{sec:DS} presents statistics for our multi-rationale datasets.

\subsection{Evaluation Metrics}\label{subsubsec:EM}

Based on the single-rationale metrics proposed in \cite{Mao2023Doubly}, we propose the following four metrics for our multi-rationale benchmark per image:

\begin{itemize}
    \item $\textbf{RR} \in [0,1]$: $1(\text{category is right})\cdot \text{accuracy}_{\text{rationales}}$,
    \item $\textbf{RW} \in [0,1]$: $1(\text{category is right}) \cdot (1 - \text{accuracy}_{\text{rationales}})$,
    \item $\textbf{WR} \in [0,1]$: $1(\text{category is wrong})\cdot \text{accuracy}_{\text{rationales}}$,
    \item $\textbf{WW} \in [0,1]$: $1(\text{category is wrong})\cdot (1 - \text{accuracy}_{\text{rationales}})$,
\end{itemize}
where $1(\cdot)$ is the indicator function and $\text{accuracy}_{\text{rationales}}=\frac{|\text{correct rationales}|}{|\text{ground-truth rationales}|}$. The four values sum to $100\%$, and the goal is to maximise \textbf{RR} (i.e., categories and rationales are identified correctly).

\section{Method}\label{sec:method}

\subsection{Preliminaries: CLIP}\label{subsec:pr}
Contrastive Language-Image Pretraining (CLIP)~\cite{radford2021learning} is a widely used VLM trained on $400$ million (image, text) pairs. It consists of an image encoder and a text encoder, both of which project their inputs into a shared $d$-dimensional unit sphere in $\mathbb{R}^d$. The similarity between an image and a text is computed via the dot product $\odot$ between their respective embeddings.

In image classification, given an image embedding $x$ (obtained from CLIP's image encoder) and a predefined set of $N$ category embedding $\mathcal{C}=\left\{c_i\right\}_{i\in [N]}$ (obtained from CLIP's text encoder), CLIP selects the category $\hat{c}$ as follows:

\begin{equation}
    \tilde{c} = \text{argmax}_{c\in \mathcal{C}}\  x \odot c{.}\label{eq:1}
\end{equation}
 The conditional distribution $P(c|x)$ can be defined using a softmax function:
\begin{align}\label{eq:vci}
    P(c|x) = \frac{\exp\left( \tau\,\, c\odot x \right)}{\sum_{c'\in \mathcal{C}}\exp\left(\tau\,\, c'\odot x \right)}{,}
\end{align}
where $\tau$ is a temperature parameter that controls the sharpness of the distribution: lower values yield smoother distributions, while higher values produce sharper peaks. In uncertainty calibration tasks, it is common to use lower values of $\tau$~\cite{guo2017calibration}.

\begin{figure}[t]
    \centering
    \includegraphics[width=\linewidth]{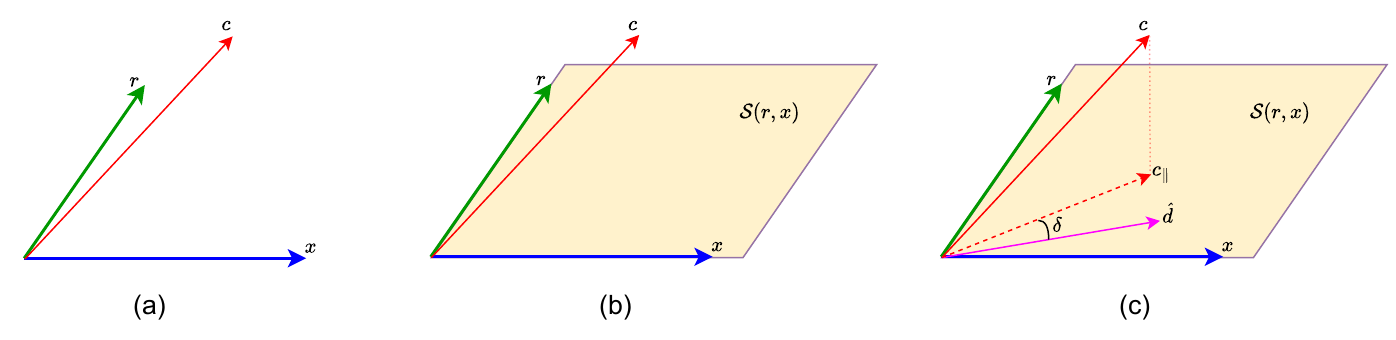}
    \caption{Illustration of Contrastive Conditional Inference (CCI). (a) We consider a triplet $(x, c, r)$ and aim to estimate $P(c|x, r)$. (b) Constructing the hyperplane $S(x, r)$. (c) Projecting $c$ onto the hyperplane, denoted $c_\parallel$, and computing its alignment with a desirable direction $\hat{d}$ using $|c_\parallel|\cos \delta$. This is followed by a softmax step to compute $P(c|x, r)$.}
    \label{fig:1}
\end{figure}

\subsection{Conditional Inference}

\subsubsection{Motivation}

As discussed earlier, previous methods \cite{Mao2023Doubly, ecor} model conditioning on more than one condition via the term \textit{"because"} in the input textual prompt: \texttt{A photo of a <$c$> \textbf{because} there is <$\hat{r}$>}.
Due to the limitations of CLIP's text encoder, these approaches lack a probabilistic formulation of conditioning, especially when handling multiple conditions \cite{zhang2024long}. 
To show this, we examine whether they satisfy Bayes’ rule. Given a ground-truth triplet (image, category, rationale), i.e., $(x, \hat{c}, \hat{r})$, we define the ratio $RT$ as follows:

\begin{equation}
    RT = \frac{P(\hat{c}|x)P(\hat{r}|\hat{c},x)}{P(\hat{r}|x)P(\hat{c}|\hat{r},x)} {.} \label{equation:RT}
\end{equation}
Since the distributions $P(\hat{c}|x)$ and $P(\hat{r}|x)$ can be accurately estimated from the vanilla CLIP inference (Eq.~\ref{eq:vci}), the ratio $RT$ would be close to $1$, iff the conditional distributions $P(\hat{c}|\hat{r},x)$ and $P(\hat{r}|\hat{c},x)$ were accurately estimated. Table \ref{tab:conditioning} shows $RT$ values for several datasets 
(lower temperatures are more valid for this experiment, see Section \ref{subsec:pr}), showing that the use \textit{"$<$because$>$"} fails to model conditioning.

\begin{table}[t]
    \centering
    \footnotesize
    \caption{$RT$ values as defined in Eq.~\ref{equation:RT} for different datasets and temperatures (best results in each setting in bold). Values that are closer to 1 are better and demonstrate a more accurate formulation of conditioning.}
    \label{tab:conditioning}
    
    \begin{tabular}{cc|rrrr}
        \toprule

        \textbf{Temperature $\tau$} & \textbf{Method} & \textbf{CIFAR-10$^\dagger$} & \textbf{CIFAR-100$^\dagger$} & \textbf{Food-101$^\dagger$}  & \textbf{SUN$^\dagger$}\\
        \midrule

        \multirow{2}{*}{0.5} & $<$because$>$ & 5.675 & 4.760 & 4.455  & 4.336\\
        & Ours & \textbf{1.004} & \textbf{1.007} & \textbf{1.011} & 
        \textbf{1.009}\\
        \midrule
        
        \multirow{2}{*}{1.0} & $<$because$>$ & 5.752 & 4.839 & 4.559 & 4.397\\
        & Ours & \textbf{1.008} &  \textbf{1.014} & \textbf{1.023} & \textbf{1.018}\\
        \midrule
        
        \multirow{2}{*}{10} & $<$because$>$ & 7.522 & 6.892 & 7.294 & 5.994\\
         & Ours & \textbf{1.101} & \textbf{1.228} & \textbf{1.414} & \textbf{1.356}\\
        \midrule
        
        \multirow{2}{*}{20} & $<$because$>$ & 10.491 & 11.486 & 13.696 & 9.594\\
         & Ours & \textbf{1.271} & \textbf{1.709} & \textbf{2.549} & \textbf{2.389}\\
        \midrule
        
        \multirow{2}{*}{50} & $<$because$>$ & 33.230 & 68.093 & 98.228  & 66.231\\
         & Ours & \textbf{1.854} & \textbf{4.428} & \textbf{7.810} & \textbf{25.755}\\
        \bottomrule
    \end{tabular}

\end{table}

 \subsubsection{Contrastive Conditional Inference}
 \begin{definition}[{Conditioning Problem}]
     Given a target embedding $c$, an image embedding $x$ and its corresponding rationales embeddings $\left\{\hat{r_i}\right\}_{i\in [m]}$, the goal is to estimate the conditional probability $p(c|x, \hat{r_1},\hat{r_2},\dots,\hat{r_M})$.
 \end{definition}
 

Before proposing our solution, we make the following assumption:
\begin{assumption}\label{assump:1}
    For a ground-truth $(\hat{c}, \hat{x}, \{\hat{r}_i\}_{i\in [M]})$, the embedding $\hat{c}$ should be as close (dot product) as possible to all rationales $\{\hat{r}_i\}_{i\in [M]}$ and image $x$.
\end{assumption}

This is a generalisation of the CLIP training objective, where relevant concepts (image, text) should be as close as possible. Considering this assumption, the following theorem gives a necessary condition for the configuration of embeddings $(\hat{c}, \hat{x}, \{\hat{r}_i\}_{i\in [M]})$.

\begin{theorem}\label{th:1}
    If embedding $\hat{c}$ is as close as possible to all embeddings $x$ and $\{\hat{r}_i\}_{i\in [M]}$, then it lies on the hyperplane of $x$ and $\{r_i\}_{i\in [M]}$ defined as $S(x,\hat{r}_1,\dots, \hat{r}_M)$.
\end{theorem}

The proof is given in Appendix \ref{subsec:proof1}. Now, given above theorem and a set of ground-truth rationales for an image, we project every category embedding $c\in \mathcal{C}$ on $S(x,\hat{r}_1,\dots, \hat{r}_M)$, and check which of them has the largest projection. Since the direction of the projected category matters as well, we additionally define a desirable direction in $S(x,\hat{r}_1,\dots, \hat{r}_M)$:


\begin{definition}[Desirable Direction]\label{def:DD}
    We call a direction $d_1$ to be more desirable than $d_2$, i.e, $d_1 \succeq d_2$, if $d_1\odot x \geq d_2 \odot x$, and $d_1\odot \hat{r}_i \geq d_2 \odot \hat{r}_i,\,\,\, \forall i\in [M]$.
\end{definition}

Now, we consider a coefficient vector $W=[w_0,w_2,\dots, w_M]$, where $\tilde{W}\subseteq W$, include all $w_i \le 0$. Every possible direction in the hyperplane  $S(x,\hat{r}_1,\dots, \hat{r}_M)$ can be obtained by adjusting the coefficient vector $W$. The following theorem gives the necessary condition for a desirable direction.

\begin{theorem}\label{th:2}
    
    Given a direction $d = \frac{w_0\, x + \sum_{i\in [M]}w_i\, r_i}{\left|w_0\, x + \sum_{i\in [M]}w_i\, r_i\right|}$ in hyperplane $S(x,\hat{r}_1,\dots, \hat{r}_M)$. One can always get a more desirable direction (Definition \ref{def:DD}) by inverting $w_j\in \tilde{W}$.
\end{theorem}

The proof is given in Appendix \ref{subsec:proof2}. This theorem states that a desirable direction is a positive weighted average of embeddings $x$ and $\{r_i\}_{i\in [M]}$. This can be seen in vanilla CLIP: there is one condition, and the desirable direction is precisely the condition (not its opposite). 

\begin{algorithm}[tb]
\caption{Contrastive Conditional Inference (CCI)}
\label{algo:cci}

\KwIn{\(x\in\mathbf{R}^d,\,\, \{\hat{r}_i\in\mathbf{R}^d\}_{i\in [M]},\,\, \mathcal{C}=\{c\in\mathbf{R}^d\}\)}
\KwOut{\( P_{\mathcal{C}} = \left[P\left(c|\hat{r}_1,\hat{r}_2, \dots, \hat{r}_M)\right),\,\, \forall c\in \mathcal{C}\right] \)}
$S \gets [x, \hat{r}_1, \hat{r}_2, \cdots, \hat{r}_M]$;x \tcp{Hyperplane $S$}
$\hat{d} \gets \frac{x + \hat{r}_1+\dots+\hat{r}_{M}}{\left|x + \hat{r}_1+\dots+\hat{r}_{M}\right|}$ \tcp{Desirable Direction}
\For{$c \in \mathcal{C}$}{
  $z_{c} \gets (S^TS)^{-1}S^T c$ \tcp{MSE estimation of $S\,z_c = c$}
  $c_{\parallel} \gets S\,z_c$ \tcp{Projection of $c$ on $S$}
  $\text{logit}_{c} \gets c_{\parallel} \odot \hat{d}$ \tcp{Closeness of the $c_{\parallel}$ with $\hat{d}$}
}

$P_{\mathcal{C}} \gets \text{SoftMax}\left(\left\{\text{logit}_{c}\right\}_{c\in\mathcal{C}}\right)$ \tcp{Normalising logits}
\Return \( P_{\mathcal{C}} \)
\end{algorithm}

Resulting from the above theorem, all directions that are a positive weighted average of the image and the rationales can be our desired direction. To obtain precisely $w_0, w_1, \dots, w_M$, we need prior information about the importance of images and rationales for a category, which is not available in most cases. So, we consider the uniform case where all $w_i$'s are equal, which leads to 
\begin{align}
    \hat{d} = \frac{x+r_1+r_2+\dots+r_M}{|x+r_1+r_2+\dots+r_M|}{.}
\end{align}
Now, given the desirable direction in the constructed hyperplane, we simply compute the dot product of the projected $c$, i.e., $c_\parallel$, with this direction. It is important to note that the above two theorems are valid in ideal cases, but in practice, we would like to find the category that satisfies them better. Algorithm \ref{algo:cci} represents our contrastive conditional inference (CCI). Figure \ref{fig:1} shows an illustrative example of our CCI algorithm. Also, when we have only one condition (image $x$), CCI is simplified to vanilla CLIP inference ($P(c|x)$) as in Eq. \ref{eq:1}. Also, the result of the conditioning experiment (Eq. \ref{equation:RT}) is available in Table \ref{tab:conditioning}, where our method outperforms the previous approach across all datasets. Moreover, a visualisation of saliency maps for our conditioning method versus the previous approaches is available in Appendix \ref{sec:visual}.

\subsection{Model}
\subsubsection{Explainable Image Classification: Inference} \label{subsec:Model}
Our goal is to maximise the following objective \cite{ecor}:
\begin{align}
 \max_{c,r_1,\dots,r_M}\quad P(c,r_1,r_2,\dots, r_M| x) \equiv \max_{c,r_1, \dots, r_M} \quad P(r_1,r_2,\dots, r_M| x)P(c|r_1,r_2,\dots,r_M,x) \label{eq:15}
\end{align}

The term $P(c|r_1,r_2,\cdots,r_M,x)$ can be calculated by Algorithm \ref{algo:cci}. However, at the first step, we must find rationales in the image $x$, which is challenging since even with knowing the number of rationales, i.e., $M$, we need to consider $|\mathcal{R}|\choose M$ candidate sets of rationales, which is intractable for large datasets. Here, we propose Algorithm \ref{algo:fr} that follows an iterative manner for selecting rationales that maximise the joint distribution, which has both good performance and efficiency. For a deeper search, Algorithm \ref{algo:fr} can be implemented by using beam search with $K_{\text{beam}}$. After obtaining rationales, one can call Algorithm \ref{algo:cci} to find the category.
\begin{algorithm}[tb]
\label{algo:fr}
\caption{Finding Rationales}
\KwIn{\(M,\,\,\, x\in\mathbf{R}^d,\,\, \mathcal{R}=\{r\in\mathbf{R}^d\},\,\, \mathcal{C}=\{c\in\mathbf{R}^d\}\)}
\KwOut{\( \hat{R} = \left[\hat{r}_1, \hat{r}_2, \dots, \hat{r}_M\right] \)}
$R \gets [r_1, r_2, \dots, r_{|\mathcal{R}|}]$ \tcp{Matrix of all rationales in $\mathcal{R}$}
$counter \gets 0$\;
$\hat{R} \gets \text{None}$ \tcp{Matrix of selected rationales}
\For{$i \gets 1$  \KwTo $M$}{
    $R' \gets R \setminus \hat{R}$ \tcp{Do not consider selected rationales}
    $P_{R'} \gets \text{SoftMax}(x^T R')$ \tcp{Probabilities of rationales in $R'$}
    \For{$(r',p_{r'})\in (R', P_{R'})$}{
        $P_{\left.\mathcal{C}\right|r'} \gets \textsc{CCI}(x,\, \hat{R}\cup \{r'\},\, \mathcal{C})$ \tcp{Calling Algorithm \ref{algo:cci}}
        $P_{\mathcal{C},{r'}} \gets p_{r'} \times P_{\left.\mathcal{C}\right|r'}$ \tcp{Joint probability of candidate $r'$}
    }
    $\hat{r}_i, \hat{c} \gets \max_{r', c}\quad \left\{P_{\mathcal{C},r'}\right\}_{r\in R'}$ \tcp{Selecting best candidate $r'$}
    $\hat{R} \gets \hat{R} \cup \{\hat{r}_i\}$ \tcp{Appending to selected rationales}
}
\Return \( \hat{R}\)
\end{algorithm}

\subsubsection{Explainable Image Classification: Fine-tuning}
Although our approach does not require retraining CLIP, we fine-tune CLIP in some experiments. Our fine-tuning objective is the same as in Eq.~\ref{eq:15}, expressed through this loss:
\begin{align}
    \mathcal{L} = \mathbf{E}_{x}\left[-\log P(\hat{r}_1,\hat{r}_2,\dots,\hat{r}_M|x) - \log P(\hat{c}|\hat{r}_1,\hat{r}_2,\dots,\hat{r}_M,x)\right]
\end{align}
where $\{\hat{r}_i\}_{i\in [M]}, \hat{c}$ are ground-truth rationales and category of image $x$.

\section{Experiments}
We conduct comprehensive experiments to demonstrate the superiority of our proposed approach over existing methods.

\subsection{Baselines}
We compare our approach with two baselines (other similar approaches could not be well fitted to this task and setup): DROR \cite{Mao2023Doubly} and ECOR \cite{ecor}. DROR models the conditional distribution $P(c|r, x)$ using the prompt structure \textit{"This is a photo of a $<$category$>$ because there is $<$rationale$>$"}. In contrast, ECOR models the joint distribution $P(c, r|x) = P(r|x)P(c|r,x)$, where the conditioning component is the same as in DROR. Both methods rely on the word ``because'' for conditioning, which we have shown to be ineffective in Table \ref{tab:conditioning}.

\subsection{Implementation Details}
We employ \textit{CLIP ViT-L/14} for the CIFAR-10$^\dagger$, CIFAR-100$^\dagger$, Food-101$^\dagger$, Caltech-101$^\dagger$, and SUN$^\dagger$ datasets, and use \textit{CLIP ViT-B/32} for ImageNet$^\dagger$. For evaluating our method, we set $K_{\text{beam}}$ (Algorithm \ref{algo:fr}) to 1, based on ablation studies (see Appendix \ref{subsec:kbeam}). When evaluating DROR and ECOR on our multi-rationale benchmark, we retrieve the top-$M$ (category, rationale) pairs per sample and apply our evaluation metrics (\ref{subsubsec:EM}). For fine-tuning, we use visual prompt-tuning \cite{jia2022visual} with three shallow learnable prompts for CIFAR-10$^\dagger$, CIFAR-100$^\dagger$, Food-101$^\dagger$, and Caltech-101$^\dagger$, 100 shallow prompts for SUN$^\dagger$, and 30 deep learnable prompts for ImageNet$^\dagger$. We use a learning rate of 0.004 with a CosineAnnealingLR scheduler \cite{loshchilov2016sgdr}, including a 20\% warmup phase.

\begin{table}[tb]
    \centering
    \footnotesize
    \label{tab:eval_individual}
    \caption{Performance on individual datasets. \textbf{RR} (correct category and rationale) is the only metric to maximise; \textbf{RW}, \textbf{WR}, and \textbf{WW} are error types to minimise. Best results per setting are in bold. ZS = Zero-Shot, FT = Fine-Tuned.}
    
      \begin{tabular}{l l a|a c|c c|c c|c}
        \toprule
        \multirow{2}{*}{\textbf{Dataset}} & \multirow{2}{*}{\textbf{Method}} & \multicolumn{2}{c}{\textbf{RR} $\uparrow$} & \multicolumn{2}{c}{\textbf{RW} $\downarrow$} & \multicolumn{2}{c}{\textbf{WR} $\downarrow$} & \multicolumn{2}{c}{\textbf{WW} $\downarrow$}\\
        & & ZS & FT & ZS & FT& ZS & FT & ZS & FT \\
        \midrule
        \multirow{3}{*}{CIFAR-10$^\dagger$} & DROR & 36.94 & 20.61 & 60.62 & 77.82 & 0.67 & 0.09 & 1.78 & 1.48\\
        & ECOR & 34.02 & 20.00 & 64.14 & 77.29 & 55.04 & 0.44 & 1.28 & 2.27\\
        & Ours & \textbf{41.18} & \textbf{58.92} & 57.29 & 39.55 & 0.00 & 0.06 & 1.53 & 1.47\\
        
        \midrule
         \multirow{3}{*}{CIFAR-100$^\dagger$} & DROR & 17.99 & 9.45 & 70.02 & 78.66 & 1.50 & 0.96 & 10.49 & 10.93\\
         & ECOR & 21.41 & 11.93 & 67.56 & 74.59 & 1.40 & 1.53 & 9.63 & 11.96\\
         & Ours & \textbf{22.56} & \textbf{28.48} & 46.12 & 50.54 & 3.18 & 2.22 & 28.14 & 18.77\\
         
        \midrule
         \multirow{3}{*}{Food-101$^\dagger$} & DROR & 15.21 & 8.10 & 70.90 & 76.91 & 1.63 & 0.69 & 12.26 & 14.31\\
         & ECOR & 21.88 & 11.72 & 65.41 & 73.05 & 1.80 & 1.25 & 10.92 & 13.80\\
         &  Ours &  \textbf{22.15} & \textbf{26.26} & 39.91 & 44.37 & 3.23 & 3.86 & 34.70 & 25.52\\
    
        \midrule
         \multirow{3}{*}{Caltech-101$^\dagger$}  & DROR & 19.81 & 11.94 &	68.81 & 75.56 &	1.90 & 1.06 & 9.48 & 11.44\\
        & ECOR & \textbf{24.85} & 14.35 & 64.38 & 70.27 & 1.91 & 2.19 &	8.86 & 13.20\\
        & Ours & 23.00 & \textbf{28.80} & 36.40 & 41.15 & 4.46 & 2.85 &	36.15 & 27.19\\
        
        \midrule
         \multirow{3}{*}{SUN$^\dagger$} & DROR&	11.96 & 5.82 	&46.14 & 44.39	&5.18 & 2.53	&38.73 & 42.26\\
         & ECOR &  12.78 & 6.59	&45.00 & 45.06 &	5.42 & 3.67	&36.80 & 44.68\\
         & Ours & \textbf{12.84} & \textbf{15.73}	&33.00 & 34.56	&4.84 & 4.35	&49.32 & 45.36\\
    
        \midrule
         \multirow{3}{*}{ImageNet$^\dagger$} & DROR & 2.92 & 4.75 & 56.95 & 55.58 & 1.88 & 2.22 & 38.26 & 37.45\\
         & ECOR & 4.83 & 6.99 & 51.39 & 49.43 & 3.29 & 3.59 & 40.48 & 39.99\\
         & Ours & \textbf{8.35} & \textbf{12.61} & 20.82 & 28.86 & 3.92 & 3.58 & 66.91 & 54.96\\
        \bottomrule
  \end{tabular}
\end{table}

\subsection{Results}
\subsubsection{Zero-shot Results}
The four ``ZS'' columns in Table \ref{tab:eval_individual} show the zero-shot results. While our method does not require any training for zero-shot evaluation, DROR and ECOR must be trained on each dataset from the DROR benchmarks (single rationale) before evaluation on our multi-rationale benchmark. Since the DROR benchmark includes all the rationales present in ours, DROR and ECOR benefit from prior exposure to the rationales and categories used in our benchmark. Despite this advantage, our approach outperforms both methods in the \textbf{RR} metric (the primary metric of interest) across all datasets except Caltech-101, where it ranks second. These results highlight the generalisability and zero-shot capability of our method, which can be deployed without fine-tuning while still achieving reliable performance.

\subsubsection{Fine-tuning Results}\label{subsubsec:FT}
Table \ref{tab:eval_individual} also presents fine-tuned results on the multi-rationale benchmark (``FT'' columns). As shown, our method’s performance improves across all datasets compared to its zero-shot counterpart. In contrast, the performances of DROR and ECOR degrade significantly, indicating their limitations when fine-tuned on a multi-rationale benchmark. These methods are unable to jointly model multiple rationales for a single image and instead attempt to optimise each rationale independently, which leads to suboptimal results.

\subsection{Cross-Dataset Evaluation}
To further assess the generalisability of the methods, Table \ref{tab:cross_eval} presents cross-dataset evaluation results, where models are trained on one dataset and tested on others. Our zero-shot results remain unchanged from Table \ref{tab:eval_individual} in the paper, as no training is involved. However, for ECOR and DROR, the zero-shot results vary due to their training on corresponding single-rationale datasets. Among the datasets, Food-101 serves as an out-of-distribution (OOD) dataset due to its narrow focus on food categories, while the others cover more general domains. Our method achieves state-of-the-art performance in both zero-shot and fine-tuned settings, with only minor degradation compared to Table \ref{tab:eval_individual} in the paper, demonstrating strong transferability even after fine-tuning.

\begin{table}[t]
    \centering
    \footnotesize
    \label{tab:cross_eval}
    \caption{Cross-dataset evaluation results. Performance on individual datasets. \textbf{RR} (correct category and rationale) is the only metric to maximize; \textbf{RW}, \textbf{WR}, and \textbf{WW} are error types to minimize. Best results per setting are in bold. ZS = Zero-Shot, FT = Fine-Tuned.}
    
      \begin{tabular}{l l a|a r|r r|r r|r}
        \toprule
        \multirow{2}{*}{\begin{tabular}[c]{@{}l@{}}\textbf{Train Dataset}\\ → \textbf{Test Dataset}\end{tabular}} & \multirow{2}{*}{\textbf{Method}} & \multicolumn{2}{c}{\textbf{RR} $\uparrow$} & \multicolumn{2}{c}{\textbf{RW} $\downarrow$} & \multicolumn{2}{c}{\textbf{WR} $\downarrow$} & \multicolumn{2}{c}{\textbf{WW} $\downarrow$}\\
        & & ZS & FT & ZS & FT& ZS & FT & ZS & FT \\
        \midrule
        
        \multirow{3}{*}{\begin{tabular}[c]{@{}l@{}}CIFAR-100$^\dagger$\\ → CIFAR-10$^\dagger$\end{tabular}} & DROR  & 31.64 & 14.50 &	65.30 & 82.27& 0.67 & 0.26 & 2.39 & 2.97\\
        & ECOR & 32.38 & 15.90 & 64.87 & 80.26& 88.18 &0.79 &	1.87 & 3.06\\
        & Ours & \textbf{41.18} & \textbf{43.92}	& 57.29 & 54.25 & 0.00 & 0.08 & 1.53 & 1.76\\
        \midrule
        
        \multirow{3}{*}{\begin{tabular}[c]{@{}l@{}}CIFAR-100$^\dagger$\\ → Food-101$^\dagger$\end{tabular}} & DROR  & 3.50 & 1.15 &	76.23 & 73.91	&1.43 & 0.63 &	18.84 & 24.31\\
        & ECOR  & 15.24 & 3.73&	64.57& 60.43	&2.71& 2.15	& 17.47 & 33.68\\
        & Ours & \textbf{22.15} & \textbf{18.38}	& 39.91 & 37.98 &	3.23 & 3.21	&34.70 & 40.43\\
        \midrule
        
        \multirow{3}{*}{\begin{tabular}[c]{@{}l@{}}CIFAR-100$^\dagger$\\ → Caltech-101$^\dagger$\end{tabular}} & DROR  & 13.19&5.54	&72.34&76.88&	2.05&1.11&	12.42& 16.47\\
        & ECOR  & 20.74 & 10.54 &	65.61 & 68.33 & 	2.52 & 2.59 & 	11.13 & 18.54\\
        & Ours & \textbf{23.00} & \textbf{23.44} &	36.40 & 39.44 &	4.46 & 4.14 & 	36.15 & 32.98\\
        \midrule
        
        \multirow{3}{*}{\begin{tabular}[c]{@{}l@{}}Caltech-101$^\dagger$\\ → CIFAR-10$^\dagger$\end{tabular}} & DROR  & 28.33 & 13.89 &	66.78 & 80.96 &	1.70 & 0.87 &	3.20 & 4.28\\
        & ECOR  & 31.55 & 16.33 &	65.08 & 79.48 &	0.88 & 0.79 &	2.48 & 3.41\\
        & Ours & \textbf{41.18} & \textbf{43.64} &	57.29 & 52.38 &	0.00 & 0.15 & 1.53 & 3.82\\
        \midrule
        
        \multirow{3}{*}{\begin{tabular}[c]{@{}l@{}}Caltech-101$^\dagger$\\ → CIFAR-100$^\dagger$\end{tabular}} & DROR  & 9.06 & 4.32 &	74.91 & 75.35 &	2.18 & 1.94 &	13.84 & 18.39\\
        & ECOR  & 16.31 & 7.02 &	67.67 & 71.58 &	2.28 & 2.33 &	13.75 & 19.07\\
        & Ours & \textbf{22.56} & \textbf{23.87} & 46.12 & 50.38 & 3.18 & 2.07 & 	28.14 & 23.69\\
        \midrule

        \multirow{3}{*}{\begin{tabular}[c]{@{}l@{}}Caltech-101$^\dagger$\\ → Food-101$^\dagger$\end{tabular}} & DROR  & 3.89	& 1.23 & 75.05 & 77.24 &	1.61 & 0.81 &	19.46 & 20.72\\
        & ECOR  & 14.50 & 4.66 &	63.47 & 70.85 &	2.57 & 1.44 & 	19.47 & 23.06\\
        & Ours & \textbf{22.15} & \textbf{19.96} & 	39.91 & 39.08 & 	3.23 & 3.10 & 	34.70 & 37.85\\
        \midrule
        
        \multirow{3}{*}{\begin{tabular}[c]{@{}l@{}}SUN$^\dagger$\\ → CIFAR-10$^\dagger$\end{tabular}} & DROR  & 25.66 & 10.83 &	71.29 & 86.29 & 	1.00 & 0.26 & 	2.06 & 2.62\\
        & ECOR  & 30.41 & 13.97 &	67.76 & 82.88 &	0.49 & 0.52 & 	1.35 & 2.62\\
        & Ours & \textbf{41.18} & \textbf{39.23} &	57.29 & 57.10&	0.00 & 0.00 & 	1.53 & 3.67\\
        \midrule
        
        \multirow{3}{*}{\begin{tabular}[c]{@{}l@{}}SUN$^\dagger$\\ → CIFAR-100$^\dagger$\end{tabular}} & DROR & 10.49 & 4.55 & 	72.33 & 75.36 & 	1.88 & 1.52 & 	15.30 & 18.56\\
        & ECOR  & 15.58 & 6.86 & 	68.36 & 70.56 &	2.09 & 1.85 & 	13.97 & 20.72\\
        & Ours & \textbf{22.56} & \textbf{20.59} &	46.12 & 45.05 &	3.18 & 2.73 & 	28.14 & 31.62\\
        \midrule
        
        \multirow{3}{*}{\begin{tabular}[c]{@{}l@{}}SUN$^\dagger$\\ → Food-101$^\dagger$\end{tabular}} & DROR  & 5.57 & 2.08 & 	66.98 & 65.82 &	2.71 & 1.87 &	24.73 & 30.23\\
        & ECOR  & 13.16 & 4.62 & 	62.21 & 56.24 & 	2.69& 2.32 &	21.94 & 36.83\\
        & Ours & \textbf{22.15} & \textbf{16.59} &	39.91 & 33.43 & 	3.23 & 2.94 & 	34.70 & 47.03\\
        \midrule

        \multirow{3}{*}{\begin{tabular}[c]{@{}l@{}}SUN$^\dagger$\\ → Caltech-101$^\dagger$\end{tabular}} & DROR  & 11.10 & 4.97 &	71.55 & 74.40 & 	2.26 & 1.24 & 	15.09 & 19.39\\
        & ECOR  & 19.28 & 8.84 &	63.62	& 69.31 & 2.99 & 2.03 & 	14.10 & 19.82\\
        & Ours & \textbf{23.00} & \textbf{19.23} & 	36.40 & 38.64 & 	4.46 & 3.58 &	36.15 & 38.55\\
        \midrule

        \multirow{3}{*}{\begin{tabular}[c]{@{}l@{}}ImageNet$^\dagger$\\ → CIFAR-10$^\dagger$\end{tabular}} & DROR  & 11.44 & 14.24 & 82.88 & 80.79 & 0.87 & 0.70 & 4.80 & 4.28\\
        & ECOR  & 13.80 & 14.32 & 79.30 & 77.64 & 1.22 & 0.70 & 5.68 & 7.34\\
        & Ours & \textbf{41.18} & \textbf{42.97} &	57.29 & 53.67 &	0.00 & 0.67 & 	1.53 & 2.69\\
        \midrule

         \multirow{3}{*}{\begin{tabular}[c]{@{}l@{}}ImageNet$^\dagger$\\ → CIFAR-100$^\dagger$\end{tabular}} & DROR  & 6.38 & 7.29 & 70.16 & 68.64 & 2.10 & 1.96 & 21.36 &  22.11\\
        & ECOR  & 6.93 & 8.52& 64.98 & 64.15 & 2.88 & 2.67 & 25.20 & 24.66 \\
        & Ours & \textbf{22.56} & \textbf{23.71} &	46.12 &  45.04 &	3.18 &  3.17 & 	28.14 & 28.08\\
        \midrule
        
         \multirow{3}{*}{\begin{tabular}[c]{@{}l@{}}ImageNet$^\dagger$\\ → Food-101$^\dagger$\end{tabular}} & DROR  & 1.66 & 1.92 & 68.38 & 60.92 & 1.28 & 1.58 & 28.68 & 35.59\\
        & ECOR  & 5.83 & 4.52 & 64.16 & 54.16 & 1.88 & 2.00 & 28.13 &  39.32\\
        & Ours & \textbf{22.15} & \textbf{13.12} &	39.91 &  30.90 & 	3.23 &  3.91 & 	34.70 & 52.07\\
        \midrule
        
         \multirow{3}{*}{\begin{tabular}[c]{@{}l@{}}ImageNet$^\dagger$\\ → Caltech-101$^\dagger$\end{tabular}} & DROR  & 7.45 & 8.88 & 69.93 & 67.64 & 2.41 & 2.34 & 20.21 & 21.14\\
        & ECOR & 9.95 & 10.55 & 64.34 & 62.11 & 2.99 & 3.35 & 22.72 & 23.98\\
        & Ours & \textbf{23.00} & \textbf{22.76} & 	36.40 &  37.46 & 	4.46 & 4.43 &	36.15 & 35.35\\
        \midrule
        
        \multirow{3}{*}{\begin{tabular}[c]{@{}l@{}}ImageNet$^\dagger$\\ → SUN$^\dagger$\end{tabular}} & DROR  & 1.43 & 1.34 & 32.69 & 29.66 & 3.04 & 2.97 & 62.84 & 66.03\\
        & ECOR  & 2.15 & 6.99 & 28.32 & 49.43 & 4.17 & 3.59 & 65.35 & 39.99\\
        & Ours & \textbf{12.84} & \textbf{10.53} &33.00 & 27.42 & 4.84 & 3.96	& 49.32 & 58.09\\
        \bottomrule
  \end{tabular}
   
\end{table}

\section{Conclusion}
In this work, we introduced a new multi-rationale benchmark and a training-free Contrastive Conditional Inference (CCI) framework that substantially advances explainable object recognition using CLIP. Our approach enables stronger conditioning on rationales and achieves state-of-the-art performance across various settings, including zero-shot evaluation.

Our CCI framework is not tied to a specific modality and can be extended to other conditioning tasks such as compositional image retrieval (CIR), where the conditions include a query text and a reference image, and the target is an image, unlike our task, where the target is a text. Since our current framework is limited to contrastive models like CLIP, a promising future direction is to explore the applicability of our method to generative vision-language models, such as LLAVA.

\clearpage
\bibliographystyle{plainnat}
\bibliography{references.bib}

\clearpage
\appendix

\section{Proofs}
\subsection{Proof of Theorem 1}\label{subsec:proof1}
Given a unit query embedding $c\in \mathbf{R}^d$ and unit condition embeddings, $\hat{r_0}, \hat{r}_1, \hat{r}_2, \dots, \hat{r}_M$, where $\hat{r}_0 \equiv x$ for the sake of notation. We provide Theorem 1:
\begin{theorem}\label{th:1}
    If embedding $c$ is as close as possible to all embeddings $\{\hat{r}_i\}_{i\in [M]}$ (in terms of $\odot$), then it lies on the hyperplane of $\{\hat{r}_i\}_{i\in [M]}$.
\end{theorem}

\begin{proof}
    Define hyperplane $S(\hat{r}_0, \hat{r}_1, \hat{r}_2, \dots, \hat{r}_M)$, $c_\parallel$ as projected component of $c$ on this hyperplane, and $c_\perp$ as the perpendicular component. We have:

    \begin{align}
        \text{sim}(c,x_i) = c_\parallel \odot \hat{r}_i = |c_{\parallel}|\cos{\theta_{c\parallel,\hat{r}_i}} ,\quad \forall i\in [M]
    \end{align}

Where $|c_{\parallel}\leq 1|$. Now, removing the perpendicular component, $c_\perp$, we have:

\begin{align}
    c_{\parallel}^{\text{new}} &= \sqrt{c_\parallel ^ 2 + c_{\perp}^2} = 1 \\
    \text{sim}(c^{\text{new}}, \hat{r}_i) &= c_{\parallel}^{\text{new}} \odot \hat{r}_i = \cos{\theta_{c_{\parallel}, \hat{r}_i}} \geq \text{sim}(c,\hat{r}_i),\quad \forall i\in [M]
\end{align}
\end{proof}

\subsection{Proof of Theorem 2}\label{subsec:proof2}
Here we provide the proof of Theorem 2 for one rationale $\hat{r}$ (the proof for multiple rationales is analogous).

\begin{theorem}\label{th:2} 
    Given a direction $d = \frac{w_0\, x + w_1\, \hat{r}}{\left|w_0\, x + w_1\, r\right|}$ in hyperplane $S(x,\hat{r})$. If $\hat{r}\odot x \geq 0$, one can always get a more desirable direction by inverting $w_0$ or $w_1$ if $w_0\le 0$ or $w_1\le 0$.
\end{theorem}

\begin{proof}
    The assumption $\hat{r}\odot x \geq 0$ is natural since the ground-truth rationale $r$ exists in the image $x$. At the first, every direction in the $S(r,x)$ can be written as $d = \frac{w_0 x + w_1 r}{|w_0 x+ w_1 r|}$, for some $w_0, w_1$. If both $w_0, w_1 < 0$, then we have:
    \begin{align}
        \text{sim}(d, r) = \frac{w_1 + w_0 \, x\odot r}{|w_1 r + w_0 x|} \leq \frac{-w_1 - w_0\, x\odot r}{|w_1 r + w_0 x|}
    \end{align}

    Which means that inverting $w_0, w_1$ increases the similarity. The same result holds for $\text{sim}(d,x)$. Now, we prove that if one of $w_i<0$ and the other is positive, then by inverting $w_i$, we have another direction that has more similarity with $r$ and $x$. So, without loss of generality, assume that $w_1 < 0,  w_0 \geq 0$:
    \begin{align}
        \text{sim}(d, r) &= \frac{w_1 + w_0\, x\odot r}{|w_1 r + w_0 x|} \leq \frac{-w_1 + w_0\, x\odot r}{|w_1 r + w_0 x|} \label{eq:11}
    \end{align}
    The RHS of Eq.~\ref{eq:11} could be negative, and proof will be complete. But if it is positive, squaring both sides leads to:
    \begin{align}
        \frac{w_1^2 + w_0^2\, (x\odot r)^2 + 2w_1w_0\, x\odot r}{w_1^2 + w_0 ^2 + 2w_1w_0\, x\odot r} \leq \frac{w_1^2 + w_0^2\, (x\odot r)^2 - 2w_1w_0\, x\odot r}{w_1^2 + w_0 ^2 - 2w_1w_0\, x\odot r}
    \end{align}
    Multiplying both sides into $(w_1^2 + w_0 ^2 + 2w_1w_0\, x\odot r)(w_1^2 + w_0 ^2 - 2w_1w_0\, x\odot r)$ and doing some algebra simplification leads to:
    \begin{align}
        x\odot r \geq (x\odot r)^3 \label{eq:13}
    \end{align}
    which is true since $0 \leq x\odot r = \cos{\theta}\leq 1$. For $\text{sim}(d,x)$ the Eq.~\ref{eq:13} will be obtained exactly.
\end{proof}

\begin{table}[t]
    \centering
    \footnotesize
    \begin{tabular}{l | c c c c}
         \toprule
         \textbf{Dataset} & \textbf{No. of Images} & \textbf{No. of Categories} & \textbf{No. of Rationales} & \textbf{No. of Rationales / Image} \\
         \midrule 
         
         CIFAR-10$^\dagger$ & 1,773 & 10 & 56 & $\sim 3.5$\\
         CIFAR-100$^\dagger$ & 13,832 & 98 & 459 & $\sim 3.1$\\
         Food-101$^\dagger$ & 12,676 & 90 & 393 & $\sim 3.4$\\
         Caltech-101$^\dagger$ & 13,035 & 96 & 451 & $\sim 2.9$\\
         SUN$^\dagger$ & 53,975 & 396 & 1,694 & $\sim 3.3$\\
         ImageNet$^\dagger$ & 218,420 & 994 & 4,158 & $\sim 2.9$\\
         
         \bottomrule
    \end{tabular}
    \caption{Datasets Statistics}
    \label{tab:DS}
\end{table}

\begin{table}[t]
    \centering
    \footnotesize
    \begin{tabular}{l|l|accc}
         \toprule
         \textbf{Dataset} & \textbf{$K_{\text{beam}}$} & \textbf{RR $\uparrow$} & \textbf{RW $\downarrow$} & \textbf{WR $\downarrow$} & \textbf{WW $\downarrow$} \\
         \midrule

         \multirow{4}{*}{CIFAR-10} 
         & 1 & \textbf{41.18} & 56.98 & 0.00 & 1.83 \\
         & 3 & 39.61 & 52.14 & 1.57 & 6.68\\
         & 5 & 39.30 & 51.83 & 1.88 & 6.99\\
         & 10 & 39.30 & 51.83 & 1.88 & 6.99\\
         \midrule

          \multirow{4}{*}{CIFAR-100} 
         & 1 & \textbf{22.56} & 46.12 & 3.18 & 28.14 \\
         & 3 & 20.15 & 35.38 & 5.59 & 38.88 \\
         & 5 & 20.11 & 35.27 & 5.63 & 38.99 \\
         & 10 & 20.11 & 35.27 & 5.63 & 38.99\\
         \midrule
         
         \multirow{4}{*}{Food-101} 
         & 1 & \textbf{22.15} & 39.91 & 3.23 & 34.70 \\
         & 3 & 18.44 & 22.51 & 6.94 & 52.10 \\
         & 5 & 18.28 & 21.84 & 7.11 & 52.77 \\
         & 10 & 18.24 & 21.75 & 7.14 & 52.86 \\
         
         \bottomrule
          
    \end{tabular}
    \caption{Effect of $K_{\text{beam}}$ on Algorithm 2 when implemented with beam search.}
    \label{tab:k_beam}
\end{table}

\section{Multi-Rationale Benchmark} \label{sec:DS}
Table \ref{tab:DS} provides the statistics of our new multi-rationale datasets. On average, we have about $\sim 3-4$ rationales per image, which enhance the explainability and reasoning of the model for its predictions.

\section{Visualisations} \label{sec:visual}
In this section, we provide the saliency maps \cite{chefer2021generic} of our conditioning approach, Algorithm 1, and previous approaches, using the \textit{$<$because$>$} in the textual prompt. As shown in Figure \ref{fig:saliency}, our method provides more accurate maps utilizing the ground-truth rationale and category, which shows its superiority compared to previous methods.
\begin{figure}[t]
    \centering
    \includegraphics[width=\linewidth]{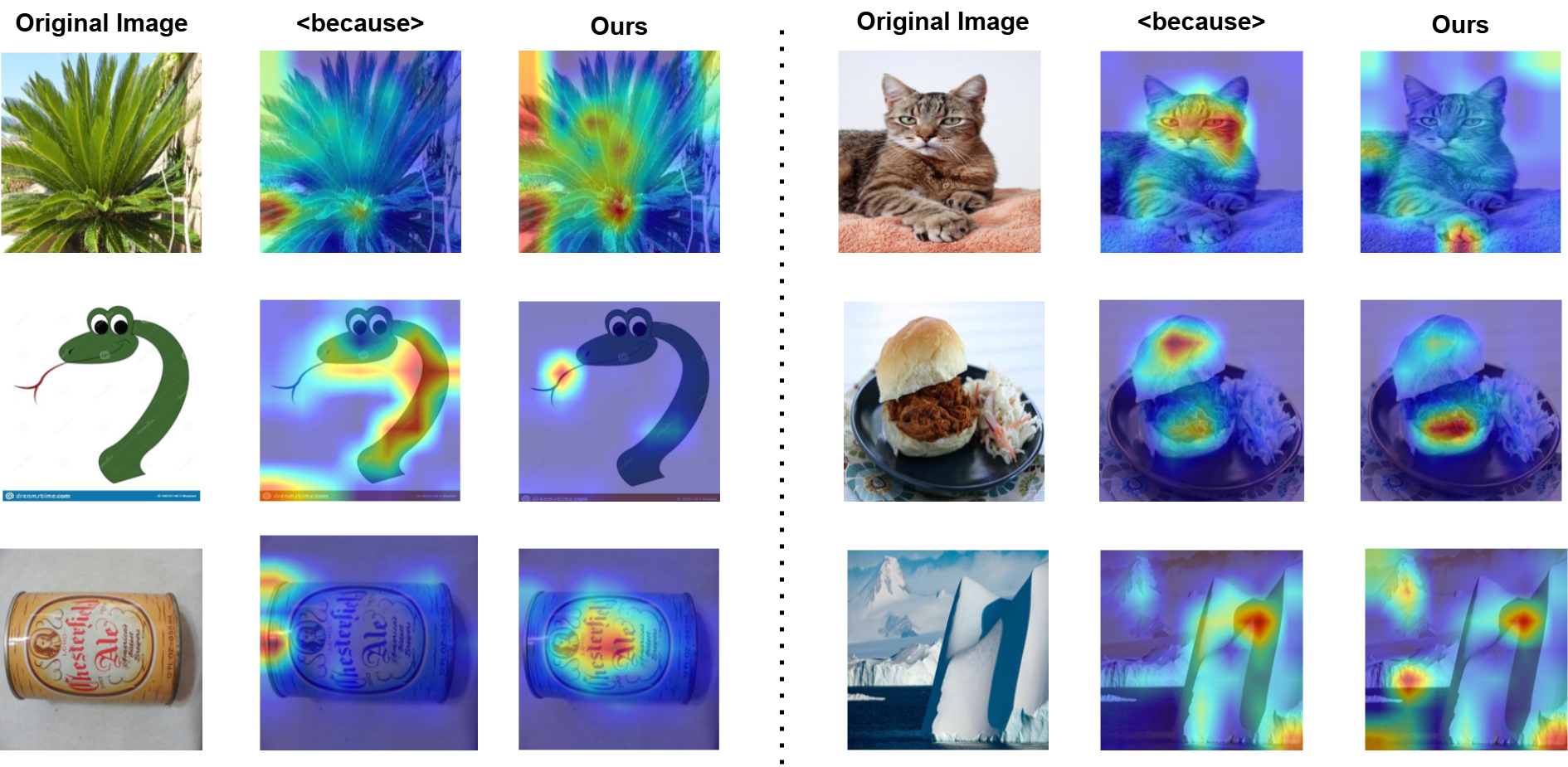}
    \caption{Saliency Maps. The left column of each part shows the original image, the middle columns show the saliency map with the prompt "A photo of a $<$category$>$ because there is $<$rationale$>$". The right column shows the saliency of the ground-truth category and rationale via our method. Ground-truth categories and rationales for the left part: $<$palm tree, long leaves$>$, and $<$snake, long tongue$>$, $<$can, information label$>$. Ground-truth categories and rationales for the right part: $<$cat, furry body$>$, $<$pulled pork sandwich, pulled pork$>$, and $<$iceberg, 
 jagged and irregular shapes of ice$>$.}.
    \label{fig:saliency}
\end{figure}

\section{Additional Experiments}

\subsection{Effect of K\textsubscript{beam}} \label{subsec:kbeam}
In this section, we analyse the effect of $K_\text{beam}$ if we implement Algorithm 2 with beam search on rationales and categories. Table \ref{tab:k_beam} represents the results for datasets CIFAR-10, CIFAR-100, and Food-101. As can be seen, with increasing $K_{\text{beam}}$, the performance (RR) degraded a bit, which is surprising. We believe the cause of this effect lies at the core of our algorithm, where in Theorem \ref{th:2}, we assume every inner product in the set of \{rationales, image\} is positive. When the size of the beam search increases, this assumption can be violated, which causes degradation in performance.

\end{document}